\documentclass{article}

\usepackage{arxiv}

\usepackage{times}
\usepackage{soul}
\usepackage{url}
\usepackage[utf8]{inputenc} 
\usepackage[T1]{fontenc}    
\usepackage{hyperref}
\usepackage[small]{caption}
\usepackage{graphicx}
\usepackage{amsmath}
\usepackage{amsthm}
\usepackage{amsfonts}
\usepackage{booktabs}
\usepackage{algorithm}
\usepackage{algorithmic}
\usepackage{bbm}
\usepackage{bm}
\newtheorem{thm}{Theorem}[section]
\newtheorem{lem}[thm]{Lemma}

\title{SoFaiR: Single Shot Fair Representation Learning}

\author{
  Xavier Gitiaux \\
  Department of Computer Science\\
  George Mason University\\
  Fairfax, VA 22030 \\
  \texttt{xgitiaux@gmu.edu} \\
   \And
 Huzefa Rangwala \\
  Department of Computer Science\\
  George Mason University\\
  Fairfax, VA 22030 \\
  \texttt{rangwala@cs.gmu.edu} \\
}

\begin{document}

\maketitle

\begin{abstract}\small\baselineskip=9pt
To avoid discriminatory uses of their data, organizations can learn to map them into a representation
that filters out information related to sensitive  attributes. However, all existing methods in fair representation learning generate a fairness-information trade-off. To achieve different points on the fairness-information plane, one must train different models. In this paper, we first demonstrate that fairness-information trade-offs are fully characterized by rate-distortion trade-offs.  Then, we use this key result and propose SoFaiR, a single shot fair representation learning method that generates with one trained model many points on the fairness-information plane.
Besides its computational saving, our single-shot approach is, to the extent of our knowledge, the first fair representation learning method that explains what information is affected by changes in the fairness / distortion properties of the representation. Empirically, we find on three datasets that SoFaiR achieves similar fairness-information trade-offs as its multi-shot counterparts.
\end{abstract}


\section{Introduction}
Machine learning algorithms increasingly support decision-making systems in contexts where outcomes have long-term implications on the subject's well-being. A growing body of evidence find that algorithms can either replicate or exacerbate existing social biases against some demographic groups. These evidence span many domains, including recidivism risk assessment \cite{ProPublica2016}, face recognition \cite{pmlr-v81-buolamwini18a}, education data mining \cite{gardner2019evaluating}, and medical diagnosis \cite{pfohl2019creating}. 

As a result, organizations that collect data are increasingly scrutinized for the potentially discriminatory use of a data by downstream applications. A flexible solution to the data-science pipeline is to control unfair uses of a data before its ingestion by a machine algorithm. Fair representation learning \cite{zemel2013learning} follows this paradigm. It is a data pre-processing method that encodes the data into a representation or code $Z$, while removing its correlations with sensitive attributes $S$. 

Current approaches in fair representation learning ~\cite{zemel2013learning,madras2018learning,gitiaux2021fair,creager2019flexibly} generate a fairness-information trade-off and are inflexible with respect to their fairness-information trade-off, which is
set at training time. This limits the deployment of fair representation learning approaches. 
For example, in medical applications, at test time, a user may need to adjust the content of the representation depending on whether gender is an appropriate feature for the downstream task at play. On one hand, for a downstream application that predicts cardiovascular risk, gender is an important/appropriate feature that should be part of the representation of the data. On the other hand, for a downstream application that predicts payment of medical bills, gender should be irrelevant to the outcome and thus, filtered out from the representation. 
With existing methods in fair representation learning, it would have to re-train a fair encoder-decoder to meet each request. At issue are computational costs and lack of consistency between released representations, since the user cannot explain what changes occur between each data product it releases.

This paper introduces SoFaiR, \textbf{S}ingle Sh\textbf{o}t \textbf{Fai}r \textbf{R}epresentation, a method to generate a unfairness-distortion curve with \emph{one single trained model}. We first show that we can derive unfairness-distortion curves from rate-distortion curves. We can control for the mutual information $I(Z, S)$  between representation and sensitive attribute by encoding $X$ into a bitstream and by controlling for its entropy. 
We then construct a gated architecture that masks partially the bitstream conditional on the value of the Lagrangian multiplier in the rate-distortion optimization problem. The mask adapts to the fairness-information trade-off targeted by the user who can explore at test time the entire unfairness-distortion curve by increasingly umasking bits. For example, in the case of a downstream medical application for which gender is sensitive and needs to be filtered out, the user sets at test time the Lagrangian multiplier to its largest value, which lowers the resolution of the representation and in a binary basis, masks the rightmost tail of the bit stream.

Besides saving on computational costs, SoFaiR allows users to interpret what type of information is affected by movement along unfairness-distortion curves. Moving upward 
unmasks bits in the tail of the bitstream and thus, increases the resolution of the representation encoded in a binary basis. By correlating these unmasked bits with data features, the practitioner has at hand a simple method to explore what information related to the features is added to the representation as its fairness properties degrade.

Empirically, we demonstrate on three datasets that at cost constant with the number of points on the curve, SoFaiR constructs unfairness-distortion curves that are comparable to the ones produced by existing multi-shot approaches whose cost increases linearly with the number of points.
On the benchmark Adults dataset, we find that increasingly removing information related to gender degrades first how the representation encodes  working hours; then, relationship status and type of professional occupations; finally, marital status.

Our contributions are as follows: (i) we formalize fairness-information trade-offs in unsupervised fair representation learning with unfairness-distortion curves and show a tractable connection with rate-distortion curves; (ii) we propose a single shot fair representation learning method to control fairness-information trade-off at test time, while training a single model; and, (iii) we offer a method to interpret how improving or degrading the fairness properties of the resulting representation affects the type of information it encodes. 

Proofs of theoretical results, additional implementation and experimental details and additional results are in the appendix. The code publicly available \href{here}{here}\footnote{See XXX}.


\paragraph{Related Work.}
A growing body of machine learning literature explores how algorithms can adversely impact some demographic groups (e.g individuals self-identified as Female or African-American) (see \cite{chouldechova2018frontiers} for a review).
This paper is more closely related to methods that transform the data into a fair representation. Most of the current literature focus on supervised techniques that  tailor the representations to a specific downstream task (e.g ~\cite{madras2018learning,edwards2015censoring,moyer2018invariant,gupta2021controllable,jaiswal2019invariant}). However, practical implementations of fair representation learning would occur in unsupervised setting where organizations cannot anticipate all downstream uses of a data. This paper contributes to unsupervised fair representation (e.g ~\cite{gitiaux2021learning}) by (i) formalizing fairness-information trade-off in a distortion-rate phase diagram, which extends compression-based approaches (e.g ~\cite{gitiaux2021fair}); and (ii), proposing an adaptive technique that allows a single trained model to output as many points as desired on a unfairness-distortion curve. 

The implementation of SoFaiR relates to approaches in rate-distortion that learn adaptive encoder and vary the compression rate at test time (e.g. ~\cite{theis2017lossy,choi2019variable}. We borrow soft-quantization techniques and entropy coding to solve the rate-distortion problem that can be derived from the fair representation learning objective. Our adaptive mask relates to the gain function in \cite{cui2020g} that selects channels depending on the targeted bit rate. 
We rely on successive refinement methods from information theory (e.g ~\cite{kostina2019successive}) that use a common encoder for all points on the unfairness-distortion curve and add new information by appending bits to a initially coarse representation. To our knowledge, we are the first contribution to implement a deep learning multi-resolution quantization and apply it to the problem of single shot fair representation learning. 
\section{Problem Statement}

\subsection{Preliminaries.}

Consider a population of individuals represented by features $X\in \mathcal{X}$ and sensitive attributes in $S\in \mathcal{S}\subset\{0, 1\}^{d_{s}}$, where $d_{s} \geq 1$ is the dimension of the sensitive attributes space. 

The objective of unsupervised fair representation learning is to map features $X\in\mathcal{X}$ into a $d-$dimensional representation $Z\in \mathcal{Z}$ such that (i) $Z$ maximizes the information related to $X$, but (ii) minimizes the information related to sensitive attributes $S$. We control for the fairness properties of the representation $Z$ via its mutual information $I(Z, S)$ with $S$. $I(Z, S)$ is an upper bound to the demographic disparity of any classifier using $Z$ as input \cite{gupta2021controllable}. 
We control for the information contained in $Z$ by constraining a distortion $d(X, \{Z, S\})$ that measures how much information is lost when using a data reconstructed from $Z$ and $S$ instead of the original $X$. Therefore, fair representation learning is equivalent to solving the following unfairness-distortion problem
\begin{equation}
\label{eq: chp5_frl}
I(D)=\min_{f}I(Z, S) \mbox{ s.t. } d(X, \{Z, S\}) \leq D
\end{equation}
where $f:\mathcal{X}\rightarrow\mathcal{Z}$ is an encoder.
The unfairness-distortion function $I(D)$ defines the minimum mutual information between $Z$ and $S$ a user can expect when encoding the data with a distortion less or equal to $D$. The unfairness-distortion problem \eqref{eq: chp5_frl} implies a fairness-information trade-off: lower values of the distortion constraint $D$ degrade the fairness properties of $Z$ by increasing $I(D)$. 
\emph{The objective of this paper is given a data $X$ to obtain the unfairness-distortion function $I(D)$ with a single encoder-decoder architecture.}

\subsection{Unfairness Distortion Curves.}
\begin{figure}
    \centering
    \includegraphics[width=0.4\textwidth]{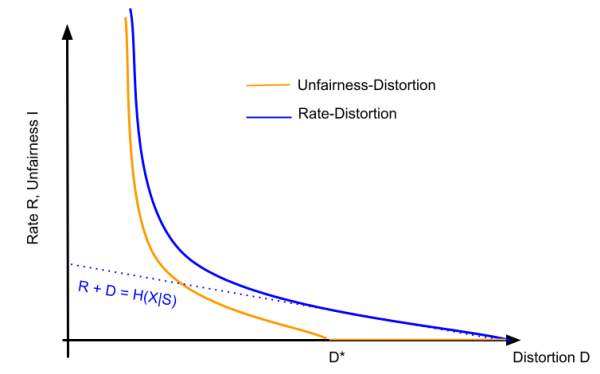}
    \caption{Unfairness-distortion curves $I(D)$ vs. rate-distortion curve $R(D)$. The unfairness distortion $I(D)$ can be deduced from the rate-distortion $R(D)$ curve by a downward shift equal to $D-H(X|S)$ if the distortion is less than $D^{*}$. }
    \label{fig: unfairness-distortion}
\end{figure}

Rate distortion theory characterizes the minimum  average number of bits $R(D)$ used to represent $X$ by a code $Z$ while the expected distortion incurred to reconstruct $X$ from the code is less than $D$. We show how to derive unfairness-distortion functions $I(D)$ from rate distortion functions $R(D)$.

\begin{thm}
\label{thm: ufd}
Suppose that the distortion is given by $d(X, \{Z, S\})=E[-\log(p(x|z, s)]$. Then, the unfairness distortion function $I(D)$ is equal to $R(D) +D -C$ if $\frac{\partial R}{\partial D}\leq -1$ and $0$ otherwise.  $C=H(X|S)$ is a constant that does not depend on $D$, but only on the data $X$. Moreover, $I(D)$ is a non-increasing convex function. 
\end{thm}

\paragraph{Phase Diagram.} Figure \ref{fig: unfairness-distortion} shows a graphical interpretation of Theorem \ref{thm: ufd} in a $(D, R)$ plane. $(D^{*}, R^{*})$ denotes the point on the rate-distortion curve where $\frac{\partial R}{\partial D} =-1$. For $D\leq D^{*}$, the rate distortion curve is above the line defined by $R+D=H(X|S)$ and that difference between $I(D)$ and $R(D)$ is  $I(Z, S)$.  For $D>D^{*}$, the rate-distortion curve is the line $R+D=H(X|S)$ and the unfairness-distortion curve is the horizontal axis.  We call the regime $D^{*}\leq D\leq H(X|S)$ the fair-encoding limit where the distortion is less than its upper limit, but $Z$ is independent of sensitive attribute $S$.

\paragraph{Information bottleneck.} Theorem \ref{thm: ufd} implies that fairness-distortion trade-offs are fully characterized by rate-distortion trade-offs. A fundamental result in rate distortion theory (\cite{tishby2000information}) shows that the rate-distortion function is given by the information bottleneck
\begin{equation}
\label{eq: ib}
R(D)=\min_{f} I(X, Z) \mbox{ s.t } d(X, \{Z, S\}) \leq D.
\end{equation}
By solving this information bottleneck with $d(X, \{Z, S\})=H(X|Z, S)$ and invoking Theorem \ref{thm: ufd}, we can recover the unfairness-distortion $I(D)$. \cite{gitiaux2021fair} provide an intuition for this result. Controlling for the mutual information $I(Z, X)$ allows to control for $I(Z, S)$ because an encoder would not waste code length to represent information related to sensitive attributes, since sensitive attributes are provided directly as an input to the decoder. We can write the information bottleneck in its Lagrangian form as
\begin{equation}
\label{eq: ib2}
\min_{f} \beta I(Z, X) + E[-\log p(x|z, s)]
\end{equation}
The coefficient $\beta$ relates to the inverse of the slope of the rate-distortion curve: $\frac{\partial R}{\partial D}= -1/\beta$. Each value of $\beta$ generates a different point along the rate-distortion curve and thus, by Theorem \ref{thm: ufd} a different point along the unfairness-distortion curve. Higher values of $\beta$ lead to representations with lower bit rate and lower mutual information with $S$. To explore a unfairness-distortion curve, existing multi-shot strategies are prohibitively expensive as they learn a new encoder $f$ for each value of $\beta$.  Moreover, they cannot interpret how changes in $\beta$ affect the representation generated by the encoder.

\begin{figure}
    \centering
    \includegraphics[width=0.45\textwidth]{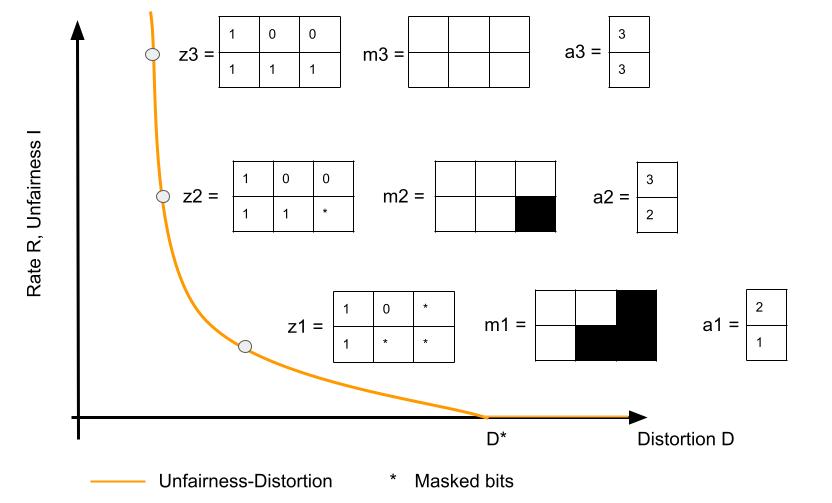}
    \caption{SoFaiR generates interpretable shifts along the unfairness-distortion curve. For a point $z1$, SoFair learns a mask $m1$ that hides bits on the tails of each dimension of the representation. By relaxing the mask to first $m2$ then $m3$, the number of bits used to represent the data increases from $a1$ to $a2$ and then $a3$; and, the representation moves to $z2$ then $z3$, which reduces the distortion at the expenses of degraded fairness properties. $z1$, $z2$ and $z3$ only differ by their masked bits (black squares). }
    \label{fig: unfairness-distortion-mask}
\end{figure}
\begin{figure*}
    \centering
    \includegraphics[width=1.0\textwidth]{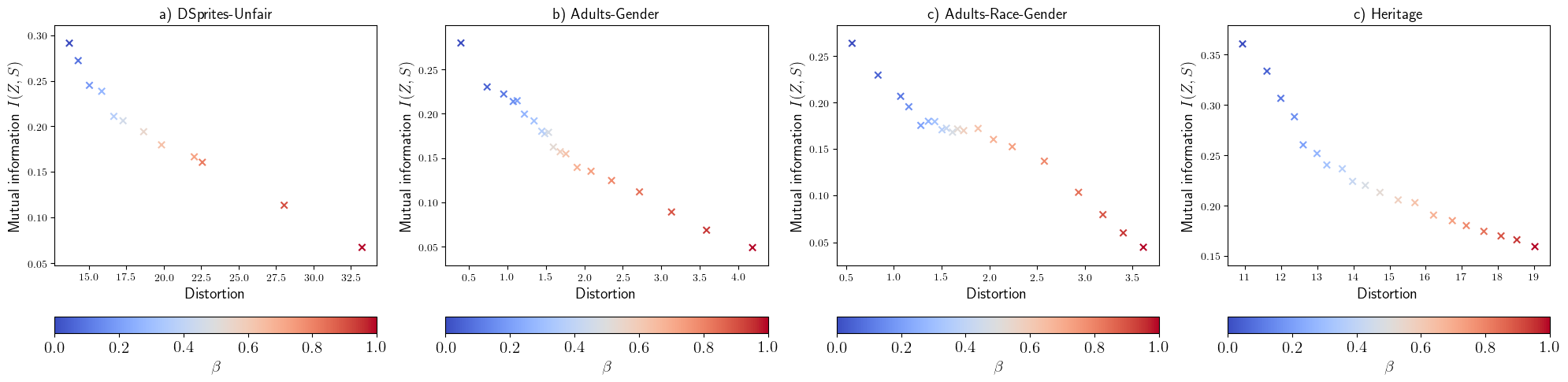}
    \caption{Unfairness-Distortion curves for a) DSprites, b) Adults-Gender,  c) Adults-Race-Gender(left) and d) Heritage. }
    \label{fig:fairness-distortion}
\end{figure*}
\section{Method: Single-Shot Unfairness-Distortion Curves.}
We propose a single-shot method, SoFaiR, to generate with one model as many points as desired on the unfairness-distortion curve.
An encoder $f:\mathcal{X}\rightarrow \{0, 1\}^{d \times r}$ common to all values of $\beta$  encodes the data into a $d$ dimensional latent variable $e\in [0, 1]^{d}$. We quantize each dimension $e_{j}$ of the $d-$dimensional latent variable with a resolution $r_{j}(\beta)$: we transform $e_{j}$ into a quantized representation $z_{j}(\beta) = [e * r(\beta)] / r(\beta)$, where $[.]$ denotes the rounding-up operation and $r(.)$ is a decreasing function of $\beta$. 

\subsection{Interpretability}
To maintain an interpretable relation between $z(\beta)$ and $z(\beta^{'})$ for $\beta^{'} <\beta$, we write $r_{j}(\beta)=2^{a_{j}(\beta)}$, where $a_{j}(.)$ is a decreasing function of $\beta$ for $j=1, ..., d$. Each dimension $z_{j}(\beta)$ of the quantized representation is then encoded into $a_{j}(\beta)$ bits. 
Moreover, for $\beta^{'}<\beta$, each dimension $j$ of the representation $z(\beta^{'})$ is made of the same $a_{j}(\beta)$ bits as $z_{j}(\beta)$, followed by $a_{j}(\beta^{'})-a_{j}(\beta)$ additional bits. Each dimension $z_{j}(\beta)$ of the quantized representation is encoded into $a_{j}(\beta)$ bits $b_{j,1}, b_{j, 2}, ..., b_{j, a_{j}(\beta)}$, where $b_{j, l}\in\{0, 1\}$ for $l=1, ..., a_{j}(\beta)$. For $\beta^{'}<\beta$ and for $j=1, ..., d$, we have
\begin{equation}\nonumber
    z_{j}(\beta^{'})= z_{j}(\beta) + \displaystyle\sum_{l=a_{j}(\beta)}^{a_{j}(\beta^{'})}b_{j, l}2^{-l}.
\end{equation}
Therefore, we have a tractable and interpretable relation between $z_{j}(\beta^{'})$ and $z_{j}(\beta)$.
This construction allows relaxing fairness constraints and decreasing distortion by unmasking additional bits for each dimension of the representation. Figure \ref{fig: unfairness-distortion-mask} shows an example for a 2-dimensional representation. A user who has released $z1$ with high distortion and low mutual information $I(Z, S)$
reduces distortion at the cost of fairness by unmasking one bit for the first dimension and two bits for the second and by generating $z2$.


\subsection{Quantization} 
We assign  a maximum number of bits $A> 0$ to encode each dimension of the representation. We apply a function $h_{e}$ to map the $d-$dimensional latent variable $e$ into $[0, 1]^{d \times A}$ and then, apply a rounding-up operator $[h_{e}(e)]$ to generate a $d\times A$ matrix, each row encoding a dimension of the representation with $A$ bits (see Figure \ref{fig: unfairness-distortion-mask} with $A=3$).  For each dimension $j$, we implement 
$a_{j}(.)$ by applying a function $h_{a}$ to map $e$ into a $d-$dimensional vector of $\mathbb{R+}^{d}$ and by computing $a_{j}(\beta) = A\left[1 - \tanh(h_{a}(e)_{j}\beta)\right)]$.

For each value of $\beta$ and each row of the matrix $[h_{e}(e)]$, we mask all the entries in position $l> a_{j}(\beta)$: for each row $j$ and each column $l$, we compute a soft mask $ m_{j, l}(\beta)= \sigma\left(a_{j}(\beta) - l\right)$
where $\sigma$ denotes a sigmoid activation; and then, we apply a rounding operator $[m_{j, l}(\beta)]$ to our soft mask.

For example, suppose that we encode in at most $A=8$ bits the embedding value $e=0.7$ and that $h(e)=e$. For $\beta=0$, we use all the bits ($a(0)=8$) and $z=0.699$; for $\beta=0.5$, $a=8(1 - \tanh((0.5)(0.7)))=5.3$ and we use only $5$ bits with $z=0.6875$.

The binarization caused by the rounding operation $[.]$ is not differentiable. We  follow ~\cite{theis2017lossy} and use a gradient-through approach that replaces $[.]$ by the identity during the backward pass of back-propagation, while keeping the rounding operation during the forward pass. 

\subsection{Entropy estimation.}
In our implementation, encoding and quantization are deterministic and $Z$ is completely determined by $X$: $H(Z|X)=0$ and $I(Z, X)=H(Z)$. 
To estimate the entropy of the representation $Z$, we use an auto-regressive factorization and write the discrete distribution  $P(z|\beta)$ as $P(z|\beta) = \prod_{j=1}^{d}P(z_{j}|z_{.<j}, \beta)$,
where the order of the dimension $j$ is arbitrary and $z_{.<j}$ denotes the dimension between $1$ and $j-1$.

We approximate the discrete distribution $p(z_{j}|z_{.<j}, \beta)$ by a continuous distribution $q(z_{j}|z_{.<j}, \beta)$ such that the probability mass of $q$ on the interval $[z_{j} - 1 / 2^{a_{j}(\beta)}, z_{j} + 1 / 2^{a_{j}(\beta)}]$ is equal to $p(z_{j}|z_{.<j}, \beta)$. We can show then that $H(z|\beta)$ is bounded above by $ \sum_{j=1}^{d} E_{z\sim p(z)}\log\int_{-1/2^{a_{j}(\beta)}}^{1/2^{a_{j}(\beta)}}q(z_{j} + u|z_{.<j}, \beta) du$ (see Appendix). We follow \cite{salimans2017pixelcnn++} and for each $j=1, ..., d$, we model  $q(.|z_{.<j}, \beta)$ as a mixture of $K$ logistic distributions with means $\mu_{j, k}(\beta)$, scales $\gamma_{j, k}(\beta)$ and mixtures probability $\pi_{j, k}(\beta)$, which allows a tractable formulation of our upper-bound (see Appendix). The resulting adaptive information bottleckneck \eqref{eq: ib2} is:
\begin{equation}
\label{eq: rd}
\min_{g, f, q, \mu, \gamma, \pi}E[-\log p(x|g(z, s, \beta)) + \beta H(z|\beta)],
\end{equation}
where $g$ is a decoder that reconstructs the data $x$ from $z$, $s$ and $\beta$. The expectation is taken over the data $x$ and values of $\beta$ uniformly drawn in $[0,1]$. 

\begin{table*}[h]
\small
    \centering
    \begin{tabular}{c|c|c|c|ccc}
      Dataset & Model   &  AUFDC &  Average per step (ms) & \multicolumn{3}{c}{Total time ($10^{6}$ ms): CPU/GPU ($\downarrow$)} \\
      & &($\downarrow$) & CPU / GPU &  4 points & 8 points & 16 points\\
      \hline
       DSprites-UnfaiR  &  SoFaiR & 0.21 & $79\pm 1.2$ / $55 \pm 0.2$ & $\bf{18.5}/\bf{13.0}$ & $\bf{18.5}/\bf{13.0}$& $\bf{18.5}/\bf{13.0}$\\
        &  SoFaiR-NOS & 0.25 & $78\pm 1.1$ / $54 \pm 0.3$ & $\bf{18.4}/\bf{13.1}$ & $\bf{18.4}/\bf{13.1}$& $\bf{18.4}/\bf{13.1}$\\
        &  MSFaiR & \textbf{0.14} &  $76\pm 3.2$ / $55 \pm 0.3$  &$71.4/ 52.1$ & $142.9/ 104.2$ & $285.8/ 208.0$ \\
         \hline
       Adults-Gender  &  SoFaiR & \textbf{0.32} &  $91 \pm 3.3 / 6  \pm 0.0$ & $\bf{2.3}/ \bf{0.1}$ & $\bf{2.3}/ \bf{0.1}$ & $\bf{2.3}/ \bf{0.1}$\\
					&  SoFaiR-NOS& 0.58 &  $91 \pm 4.3 / 6  \pm 0.0$ & $\bf{2.3}/ \bf{0.1}$ & $\bf{2.3}/ \bf{0.1}$ & $\bf{2.3}/ \bf{0.1}$\\        
        &  MSFaiR & 0.35 &  $92 \pm 1.0 / 6  \pm 0.0$ & $9.4/ 0.6$ & $18.9/ 1.1$ & $37.7/ 2.3$ \\
        \hline
       Adults-Gender-Race  &  SoFaiR & \textbf{0.30} & $92 \pm 4.3 / 6  \pm 0.0$ & $\bf{2.4}/ \bf{0.1}$ & $\bf{2.4}/ \bf{0.1}$ & $\bf{2.4}/ \bf{0.1}$\\
          &  SoFaiR-NOS & 0.53 & $92 \pm 4.0 / 6  \pm 0.0$ & $\bf{2.4}/ \bf{0.1}$ & $\bf{2.4}/ \bf{0.1}$ & $\bf{2.4}/ \bf{0.1}$\\
        &  MSFaiR &0.36 &  $90 \pm 4.0 / 6  \pm 0.0$ & $9.1/ 0.6$ & $18.3/ 1.1$ & $36.6/ 2.3$ \\
         \hline
       Heritage &  SoFaiR & 0.62 & $125 \pm 3.0/8.6\pm 1.6$  & $\bf{3.7/0.3}$  & $\bf{3.7/0.3}$ & $\bf{3.7/0.3}$ \\
        &  SoFaiR-NOS & 0.73 & $123 \pm{2.5}/10 \pm{0.3}$ & $\bf{3.7/0.3}$  & $\bf{3.7/0.3}$ & $\bf{3.7/0.3}$ \\
        &  MSFaiR & \textbf{0.56} & $123 \pm{3.1} / 10 \pm{0.8}$ & $14.7/1.2$&  $29.4/2.3$& $58.7/4.8$  \\  
    \end{tabular}
    \caption{Area under the unfairness-distortion curve and computational costs of single-shot (SoFaiR) versus multi-shot (MSFaiR) fair representation learning methods. Lower ($\downarrow$) is better. This shows that SoFaiR provides unfairness-distortion curves with similar AUFDC as MSFaiR, but at much lower computational costs.}
    \label{tab: computational_cost}
\end{table*}
\section{Experiments}
We design our experiments to answer the following research questions: 
(RQ1) Does SoFaiR generate in a single-shot unfairness-distortion curves comparable to the ones generated by multi-shot models? (RQ2) Do representations learned by SoFaiR offer to downstream tasks a fairness-accuracy trade-off on par with state-of-the-art multi-shots techniques in unsupervised fair representation learning?
(RQ3) What information is present in the additional bits that are unmasked as we move up the unfairness-distortion curve? Architecture details and hyperparameter values are in the supplementary file.

\subsection{Datasets}

We validate our single-shot approach with three benchmark datasets: \textbf{DSprite-Unfair},  \textbf{Adults}  and \textbf{Heritage}.

\textbf{DSprite Unfair} is a variant of the DSprites data 
and contains $64$ by $64$ black and white images of various shapes (heart, square, circle). 
We modify the sampling to generate a source of potential unfairness and use as sensitive attribute a variable that encodes the orientation of the shapes.

The \textbf{Adults} dataset 
contains $49K$ individuals with information on professional occupation, education attainment, capital gains, hours worked, race and marital status. We consider  as sensitive attribute, gender in \textbf{Adults-Gender}; and, gender and race in \textbf{Adults-Gender-Race}. 


The \textbf{Health Heritage} dataset 
contains $95K$ individuals with $65$ features related to clinical diagnoses and procedures, lab results, drug prescriptions and claims payment aggregated over $3$ years (2011-2013). We define as sensitive attributes an intersection variable of gender and age.

\subsection{Unfairness-distortion curves.}
To plot unfairness-distortion curves, we estimate the distortion as the  $l_{2}-$ loss between reconstructed and observed data, which is equal to  $E_{x, z, s}[-\log p(x|z, s)]$ (up to a constant) if the distribution of $p(X|Z, S)$ $X$ is an isotropic Gaussian. 
We also approximate the mutual information $I(Z,S)$ with an adversarial lower bound (see appendix):
\begin{equation}
\label{eq: lb_mutual}
    I(Z, S)\geq H(S) - \min_{c}E_{s,z}[-\log c(s|z)],
\end{equation}
where $c$ is an auditing classifier that predicts $S$ from $Z$.
Unlike adversarial methods (e.g. ~\cite{edwards2015censoring}), we do  not use this bound for training our encoder-decoder, but only for post mortem evaluation of the unfairness-distortion trade-off generated by SoFaiR. In practice, we train a set of $5$ fully connected neural networks $c: \mathcal{Z}\rightarrow\mathcal{S}$ and use their average cross-entropy to estimate the right hand side of \eqref{eq: lb_mutual}. 

\subsection{Area under unfairness-distortion curves.} To quantitatively compare the unfairness-distortion curves of competing approaches, we introduce 
the area under unfair-distortion curve, AUFDC.
A lower AUFDC means that a model achieve lower $I(Z, S)$ for a given level of distortion. 
To allow comparison across datasets, we normalize the value of AUFDC by the area of the rectangle $[0, D_{max}]\times[0, I_{max}]$, where $D_{max}$ is the distortion obtained by generating random permutation of a representation and $I_{max}$ is the value of the lower bound \eqref{eq: lb_mutual} when auditing raw data. 

\subsection{Comparative Methods.}
\paragraph{Methods.} We compare SoFaiR with five fair representation methods: 
(i) \textbf{LATFR} (e.g ~\cite{madras2018learning}) controls for $I(Z, S)$ by using the lower bound \eqref{eq: lb_mutual}
; (ii) \textbf{MaxEnt-ARL} \cite{roy2019mitigating} replaces the adversary's cross-entropy of LATFR with the entropy of the adversary's predictions; (iii) \textbf{CVIB} \cite{moyer2018invariant} replaces adversarial training with an information-theory upper bound of $I(Z, S)$; (iv)\textbf{$\beta-VAE$} \cite{higgins2016beta} solves the information bottleneck \eqref{eq: ib2} by variational inference, which upper-bounds $I(Z, S)$, provided that the decoder uses the sensitive attribute as input \cite{gitiaux2021fair}; (v) \textbf{MSFaiR} reproduces SoFaiR, but solves the rate-distortion problem \eqref{eq: rd} separately for different values of $\beta$. All methods have the  same autoencoder architecture.
Most methods are tailored to a specific downstream task. In our unsupervised setting, we repurpose them by replacing the cross-entropy of the downstream classification task with our measure of distortion $E[-\log p(x|Z, s)]$. 



\paragraph{Pareto fronts.} We construct  Pareto fronts that compare the unfairness properties of the representation to the accuracy $A_{y}$ of a downstream task classifier that predicts a downstream label $Y$ from $Z$.  Critically in our unsupervised setting, we do not provide the labels $Y$ to encoder-decoders. To match existing benchmarks, we measure the unfairness properties of the representation with the average accuracy $A_{s}$ of auditing classifiers that predict $S$ from $Z$. The higher $A_{y}$ for a given $A_{s}$, the better is the fair representation method.


\section{Results}

\begin{figure}
    \centering
    \includegraphics[width=0.5\textwidth]{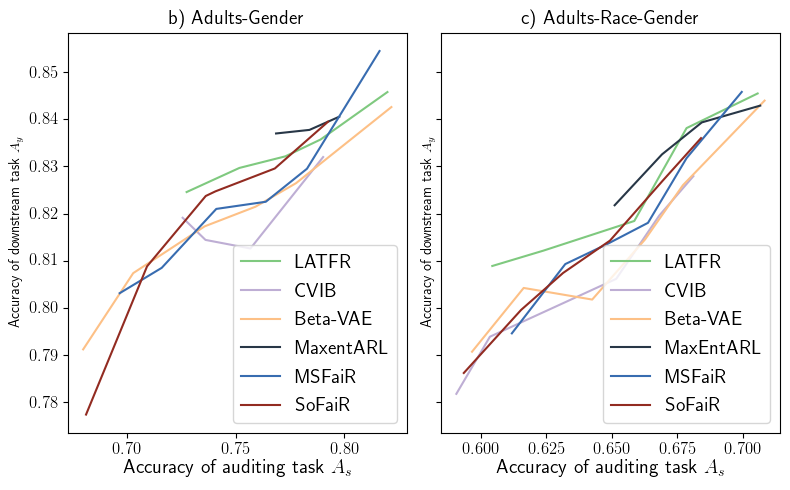}
    \caption{Pareto fronts for a) Adults-Gender (left),  b) Adults-Race-Gender(right). The downstream task label is whether income is larger than 50K. }
    \label{fig:pareto-front}
\end{figure}

\begin{figure*}[h]
    \centering
    \includegraphics[width=0.8\textwidth]{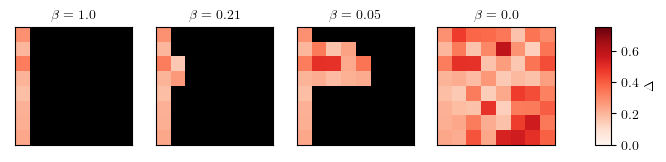}
    \caption{Unmasked bits for different values of the fairness coefficient $\beta$ for the Adults-Gender-Race dataset. Each row is a dimension of $Z$. Each colored square is an unmasked bit. Black squares represent masked bits. Darker bits exhibit higher bit demographic disparity $\Delta(b)$. As $\beta$ decreases, SoFaiR unmasks more bits for each dimension of $Z$. And, bits with higher disparity are more likely to be the last unmasked.}
    \label{fig:inter_adults}
\end{figure*}

\subsection{RQ1: Single Shot Fairness-Distortion Curves.}
Figure \ref{fig:fairness-distortion} shows SoFaiR's unfairness-distortion curves for DSprites (left), Adults-Gender (middle left), Adults-Gender-Race (middle right) and Heritage (right). By increasing at test time the value of $\beta$, the user can smoothly move down the unfairness-distortion curve: values of $\beta$ close to zero lead to low distortion - high $I(Z, S)$ points; values of $\beta$ close to one lead to higher distortion - low $I(Z, S)$ points. Figure \ref{fig:fairness-distortion} demonstrates that a solution to the adaptive bottleneck \eqref{eq: rd} allows one single model to capture different points on the unfairness-distortion curve. This result is consistent with Theorem \ref{thm: ufd} and illustrates that controlling for  the bit rate of $Z$ via its entropy $H(Z)$ is sufficient to control for $I(Z, S)$.

\paragraph{Ablation study.}
AUFDC scores in Table \ref{tab: computational_cost} show that SoFaiR is competitive with its multi-shot counterpart: SoFaiR outperforms MSFaiR for Adults-Gender and Adults-Gender-Race (lower AUFDC), but is slightly outperformed for Heritage and DSprites-Unfair (higher AUFDC).
On the other hand, SoFaiR unambigously outperforms SoFaiR-NOS, a model similar to SoFaiR but with a decoder that does not use the sensitive attribute $S$ as side-channel.
The relation between unfairness-distortion and rate-distortion curves in Theorem \ref{thm: ufd} is tractable only if we use $E[-\log(p(x|z, s)]$ as a measure of distortion and does not hold if we use  $E[-\log(p(x|z)]$ instead and the decoder does not receive $S$ as side channel.

\paragraph{Computational costs.}
Table \ref{tab: computational_cost} compares the computational costs of SoFaiR and MSFaiR. We average the cpu and gpu times of a training step 
over 10 profiling cycles and the number of training epochs. We perform the experiment on a AMD Ryzen Threadripper 2950X 16-Core Processor CPU and a NVIDIA GV102 GPU. 
The average computing cost of a training step is similar for SoFaiR and MSFaiR since both methods rely on similar architecture. However, SoFaiR's computational costs remain constant as the number of points on the unfairness-distortion curve increases, while MSFaiR's costs increase linearly. For example, 16 points for the DSprites-Unfair require about 137 hours of running time with MSFaiR and only 8 hours with SoFaiR.

\subsection{RQ2: Pareto Fronts}
In Figure \ref{fig:pareto-front}, the larger the downstream classifier's accuracy $A_{y}$ for a given value of the auditor's accuracy $A_{s}$, the better the Pareto front. 
First, SoFaiR and MSFaiR's Pareto fronts are either as good or better than the ones generated by $LATFR$, $CVIB$, $Maxent-ARL$ and $\beta-VAE$. Exceptions to this observations include Adults-Gender-Race for low values of $A_{s}$ where $LATFR$ outperforms SoFaiR/MSFaiR. Rate distortion approaches are competitive, which confirms the tight connection between rate-distortion and unfairness-distortion as presented in Theorem \ref{thm: ufd}.
Both SoFaiR and MSFaiR offer more consistent performances than $LATFR$ or $Maxent-ARL$ whose representations keep leaking information related to $S$ for Adults-Gender 
regardless of the constraints placed on the adversary. And, $\beta-VAE$ exhibits non-monotonic behavior for 
Adults-Gender. Second, Figure \ref{fig:pareto-front} shows that SoFaiR's Pareto fronts are similar to the ones offered by MSFaiR, its multi-shot counterpart. This result is consistent with AUFDC scores in Table \ref{tab: computational_cost}. Pareto fronts for Heritage and DSPrites-Unfair are in the supplementary file.

\begin{figure}[h]
    \centering
    \includegraphics[scale=0.5]{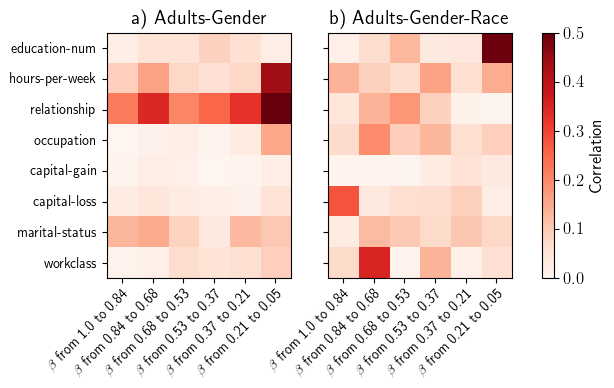}
    \caption{Additional information provided by refining the representation for Adults-Gender (left) and Adults-Gender-Race (right) dataset. This shows the correlation between data features and additional bits that SoFaiR unmasks when loosening the fairness constraint. Correlations are computed between the data features and the first principal component of newly unmasked bits. Each column corresponds to a decrease of $\beta$ as labeled on the horizontal axis.}
    \label{fig:inter_gender_info}
\end{figure} 

\subsection{RQ3: Interpretability}

\paragraph{Bit Disparity.}
We measure the disparity of each bit $b$ as $\Delta(b)=\max_{s\in \mathcal{S}}|P(b=1|S=s) - P(b=1|S\neq s)|$. Bit disparity is the demographic disparity of a classifier that returns 1 if $b=1$ and 0 otherwise. Moreover, we show in the supplementary file that $\max_{b}\Delta(b)$ is a lower bound of $I(Z, S)$: a large value of $\Delta(b)$ means that the presence of bit $b$ in the bitstream will significantly degrade the fairness properties of $Z$. In Figure \ref{fig:inter_adults}, loosening the fairness constraint at test time -- decreasing $\beta$ -- unmasks more bits, while keeping the leftmost bits identical to ones obtained with higher values of $\beta$. SoFaiR degrades gracefully the fairness properties of the representation by increasing its resolution. 
Figure \ref{fig:inter_adults} also shows that for Adults-Gender-Race, bits with higher disparity $\Delta$ are less likely to be unmasked with stringent fairness constraints -- high $\beta$ -- and are only active when more leakages related to sensitive attribute are tolerated -- low $\beta$. Therefore, by forcing SoFaiR to generate many points on the unfairness-distortion curve, we obtain an information ordering that pushes to the tail of the bitstreams the bits the most correlated with $S$. We observe a similar pattern with Adults-Gender (supplementary file).

\paragraph{Fairness and information loss.}
Unlike alternative methods in fair representation learning, SoFaiR offers a simple tool to interpret at test time what information is lost as the fairness constraint tightens. In Figure \ref{fig:inter_gender_info}, we plot for Adults-Gender and Adults-Gender-Race how additional bits unmasked as $\beta$ decreases correlate with data features. As we move up the unfairness-distortion curve for Adults-Gender, additional information first relates to marital status; then, occupation type, relationship status and hours-per-week. It means that for downstream tasks that predict marital status, a representation on the bottom right of the unfairness-distortion curve (high distortion, low $I(Z, S)$) is sufficient to achieve good accuracy. But, downstream tasks that need hours-per-week would find more difficult to obtain good accuracy without moving up the unfairness-distortion curves, i.e leaking additional information related to sensitive attribute $S$.


\section{Conclusion}
In this paper, we present SoFaiR, a single-shot fair representation learning method that allows with one trained model to explore at test time the fairness-information trade-offs of a representation of the data. Our implementation relies on a tight connection between rate-distortion and unfairness-distortion curves. SoFaiR is a step toward practical implementation of unsupervised fair representation learning approach, all the more as users can now explain what information is lost as the fairness properties of the representation improve.
\section*{Acknowledgements}
This project was supported by resources provided by the Office of Research Computing at George Mason University (URL: https://orc.gmu.edu) and funded in part by grants from the National Science Foundation (Awards Number 1625039 and 2018631).
\appendix
\section{Appendix}
\subsection{Proof of Theorem 2.1}
First, we show the following identity: 
\begin{lem}
\label{lem: 1}
$I(Z,S)=I(Z, X) + H(X|Z, S) - H(X|S)$.
\end{lem} 

\begin{proof}
The proof of Lemma \ref{lem: 1} relies on multiple iterations of the chain rule for mutual information:
\begin{equation}\nonumber
\begin{split}
    I(Z, S) &\overset{(a)}{=} I(Z, \{X, S\}) - I(Z, X|S) \\
    & \overset{(b)}{=}I(Z, X) + I(Z, S|X)  - I(Z, X|S) \\
    & \overset{(c)}{=} I(Z, X)- I(Z, X|S) \\
    & \overset{(d)}{=}I(Z, X) - I(X, \{Z, S\}) + I(X, S) \\
    & \overset{(e)}{=} I(Z, X) + H(X|Z, S) - H(X|S)
    \end{split}
\end{equation}
where $(a)$, $(b)$ and $(d)$ use the chain rule for mutual information; and, $(c)$ uses the fact that $Z$ is only encoded from $X$ and from $S$, so $H(Z|X, S)=H(Z|X)$ and $I(Z, S|X)=H(Z|X)-H(Z|X, S)=0$. And $(e)$ uses the fact that $I(X, S)=H(X) - H(X|S)$ and $I(X, \{Z, S\})=H(X) - H(X|Z, S)$. 
\end{proof}

Lemma \ref{lem: 1} implies that if the distortion is $d(X, \{Z, S\})=H(X|Z, S)$, the unfairness-distortion function is given by
\begin{equation}
\label{eq: chp5_frl_proof}
\begin{split}
I(D)= &\min_{f}I(Z, X) + H(X|Z, S) - H(X|S) \\
& \mbox{ s.t. } H(X, \{Z, S\}) \leq D \\
\end{split}
\end{equation}

Second, a fundamental theorem in rate-distortion \cite{cover2012elements} shows that if the distortion is $d(X, \{Z, S\})=H(X|Z, S)$ the rate-distortion function is given by 
\begin{equation}
\label{eq: ib_proof}
R(D)=\min_{f} I(X, Z) \mbox{ s.t } H(X|Z, S) \leq D,
\end{equation}
and that $R(D)$ is a non-increasing convex function. The next Lemma shows how solution of the minimization problem \eqref{eq: ib_proof} solves the minimization problem \eqref{eq: chp5_frl_proof} whenever $\frac{\partial R(D)}{\partial D}\leq -1$

\begin{lem}
\label{lem: 2}
Let $D\geq 0$ be a distortion value. Assume that $\frac{\partial R(D)}{\partial D}\leq -1$. A solution $f^{*}$ of the minimization \eqref{eq: ib_proof} for $D$ is also solution of \eqref{eq: chp5_frl_proof}.
\end{lem}

\begin{proof}
At the optimum, the constraint in \eqref{eq: ib_proof} is binding and thus, that $H_{f^{*}}(X|Z, S)=D$, where the subs-script $f^{*}$ reminds that the code $Z$ depends on $f^{*}$. Consider now a solution $g^{*}$ of the minimization  \eqref{eq: chp5_frl_proof} for a distortion $D$. We consider two cases: case (I) the constraint is binding for $g^{*}$ in \eqref{eq: chp5_frl_proof}; case (II) the constraint is not binding for $g^{*}$ in \eqref{eq: chp5_frl_proof}.

\textbf{case (I)}: $H_{g^{*}}(X|Z, S)=D$ and we have
\begin{equation}
\begin{split}
I(D) &= I_{g^{*}}(Z, X) + H_{g^{*}}(X|Z, S) - H(X|S)\\
     & = I_{g^{*}}(Z, X) + D -H(X|S) \\
     & \overset{(a)} {\geq} I_{f^{*}}(Z, X) + D -H(X|S), \\
     \end{split}
\end{equation}
where $(a)$ uses the fact that $f^{*}$ is solution of  \eqref{eq: ib_proof} and that $H_{g^{*}}(X|Z, S)\leq D$. Therefore, since $H_{f^{*}}(X|Z, S)\leq D$, $f^{*}$ is also solution of \eqref{eq: chp5_frl_proof}.

\textbf{case (II)}: Let denote $D^{'}$ the value of the distortion achieved by $g^{*}$. Then, $D^{'}= H_{g^{*}}(X|Z, S) < D$. We have
 \begin{equation}
 \label{eq: id}
\begin{split}
I(D) &= I_{g^{*}}(Z, X) + H_{g^{*}}(X|Z, S) - H(X|S)\\
     & = I_{g^{*}}(Z, X) + D^{'} -H(X|S) \\
     & \overset{(a)} {\geq} R(D^{'}) + D^{'} -H(X|S), \\
     \end{split}
\end{equation}
where $(a)$ follows from the definition of $R(D^{'})$. By convexity of the rate-distortion function, we have that
\begin{equation}
\label{eq: conv}
\begin{split}
R(D^{'}) - R(D) & \overset{(a)}{\geq} \frac{\partial R(D)}{\partial D} (D^{'}- D) \\
&\overset{(b)}{\geq} (D- D^{'}),
\end{split}
\end{equation}
where $(a)$ uses the convexity of $R(D)$ and that $D^{'}<D$ and $(b)$ uses that $\frac{\partial R(D)}{\partial D}\leq -1$. Hence, by combining \eqref{eq: id} and \eqref{eq: conv}, we have
\begin{equation}\nonumber
I(D) \geq R(D) +D -H(X|S) = I_{f^{*}}(Z, X) + D -H(X|S).
\end{equation}
Therefore, $f^{*}$ is also solution of the minimization \eqref{eq: chp5_frl_proof} since $H_{f^{*}}(X|Z, S)\leq D$. 
\end{proof}

It  follows from Lemma \ref{lem: 2} that we have by definition of $f^{*}$, if $\frac{\partial R(D)}{\partial D}\leq -1$
\begin{equation}\nonumber
\label{eq: id_P}
I(D) = I_{f^{*}}(Z, X) + D - H(X|S) = R(D) + D - H(X|S),
\end{equation}
which proves the first part of the statement in Theorem 2.1. Moreoover, if $\frac{\partial R(D)}{\partial D}< -1$, $\frac{\partial I(D)}{\partial D} = \frac{\partial R(D)}{\partial D} + 1 <0$, hence $I(.)$ is decreasing for $D$ such that $\frac{\partial R(D)}{\partial D} < -1$.

To prove that if $\frac{\partial R(D)}{\partial D}\geq -1$, $I(D)=0$, we first prove the following Lemma:
\begin{lem}
Let $D^{*}$ denote the value of $D$ such that $\frac{\partial R(D)}{\partial D}= -1$. For $D^{*}\geq D$, $I(D)=I(D^{*})$. 
\end{lem} 

\begin{proof}
Let $D>D^{*}$. Let $g^{*}$ be a solution of the minimization \eqref{eq: chp5_frl_proof} for $D$. Note that a solution of \eqref{eq: chp5_frl_proof}  for $D^{*}$ respects the constraint of the minimization \eqref{eq: chp5_frl_proof}  for $D$ and thus, $I(D^{*})\geq I(D)$. Let $D^{'}$ denote $H_{g^{*}}(X|Z, S)$. Then, by definition of the rate-distortion objective value \eqref{eq: ib_proof}, we have
\begin{equation}
\begin{split}
I(D) &= I_{g^{*}}(Z, X) + D^{'} - H(X|S) \\
& \geq R(D^{'}) + D^{'} - H(X|S). 
\end{split}
\end{equation}
If $D^{'}< D^{*}$, then we already know that $I(D^{'}) = R(D^{'}) + D^{'} - H(X|S)$ and that $I(D^{'}) > I(D^{*})\geq I(D)$. Moreover, by inequality \eqref{eq: id_P}, $\geq I(D^{'})$, thus $I(D^{'}) > I(D)\geq I(D^{'})$, which is a contradiction. If $D^{'}=D^{*}$, we already know that $I(D)\leq I(D^{*}) = R(D^{*}) +D^{*}-H(X|S)=I(D^{'})\leq I(D)$ and thus that $I(D)=I(D^{*})$. 

It remains to look at the case $D^{'}> D^{*}$. Consider $D^{"}\in [D^{*}, D^{'}]$. By convexity of $R(D)$ we have 
\begin{equation}\nonumber
R(D^{*}) - R(D^{'}) \leq \frac{\partial R(D^{*})}{\partial D} (D^{*} - D^{'}) \overset{(a)}{=} D^{'} - D^{*},
\end{equation}
where $(a)$ comes the fact that $\frac{\partial R(D^{*})}{\partial D}= -1$. It results that by the inequality \eqref{eq: id_P} $I(D)\geq R(D^{*}) + D^{*} -H(X|S)$. Moreover, we already know that $R(D^{*}) + D^{*} -H(X|S)=I(D^{*})$. Hence $I(D^{*})\geq I(D)\geq I(D^{*})$, which proves the equality in Lemma \ref{lem: 2}. 
\end{proof}

\begin{lem}
\label{lem: 3}
Let $D^{**}=H(X|S)$. We have $I(D^{**})=0$.
\end{lem}

\begin{proof}
Consider an encoder $g$ that generates a random variable $Z$ independent of $X$. Then $H_{g}(X|Z, S) = D^{**}$ and $I_{g}(Z, X)=0$. Therefore, $g$ respect the constraint of the minimization \eqref{eq: chp5_frl_proof} for $D^{**}$ and $I(D^{**}) \leq I_{g}(Z, X) + H_{g}(X|Z, S) - H(X|S) = 0$. Hence, $I(D^{**}) = 0$.  
\end{proof}

By combining Lemma \ref{lem: 2} and \ref{lem: 3}, we can show that $I(D)=0$ for $D\geq D^{**}$.

 \subsection{Lower Bound on \texorpdfstring{$I(Z, S)$}{I(Z, S)}}
When constructiong unfairness-distitortion curves, we approximate the mutual information $I(Z,S)$ with an adversarial lower bound. For any approximation $q(s|Z)$ of $p(s|Z)$, we have
\begin{equation}
\begin{split}
    I(Z, S)& = H(S) - H(S|Z) \\
    &= H(S) - E_{s,z}[-\log q(s|z)] + KL(p(s|z)||p(s|z) \\
    &\geq H(S) - E_{s,z}[-\log q(s|z)],
    \end{split}
\end{equation}
where the inequality comes from the non-negativity of the Kullback-Leibler divergence $KL(p|q)$. Therefore, we lower bound $I(Z, S)$ with 
\begin{equation}
H(S) - \min_{q}E_{s,z}[-\log q(s|z)],
\end{equation}
where the minimum is taken over classifiers that predict $S$ from $Z$.

\section{Additional Implementation Details}
\subsection{Entropy estimation.}
We follow a standard approach in rate-distortion \cite{theis2017lossy,choi2019variable} and approximate the discrete distribution $p(z_{j}|z_{.<j}, \beta)$ by a continuous distribution $q(z_{j}|z_{.<j}, \beta)$ such that the probability mass of $q$ on the interval $[z_{j} - 1 / 2^{a_{j}(\beta)}, z_{j} + 1 / 2^{a_{j}(\beta)}]$ is equal to $p(z_{j}|z_{.<j}, \beta)$. Therefore,
\begin{equation}
\label{eq: entr2}
    \begin{split}
        H(z|\beta) & = -\displaystyle \sum_{j=1}^{d} E\left[\log p(z_{j}|z_{.<j}, \beta)\right] \\
        & = -\displaystyle \sum_{j=1}^{d} E\left[\log\left(\int_{\frac{-1}{2^{a_{j}}(\beta)}}^{\frac{1}{2^{a_{j}}(\beta)}}q(z_{j} + u|z_{.<j}, \beta) du\right)\right] \\
        & + KL\left(p || \int_{\frac{-1}{2^{a_{j}}(\beta)}}^{\frac{1}{2^{a_{j}}(\beta)}}q(z_{j} + u|z_{.<j}, \beta) du\right) \\
        & \overset{(a)}{\leq} - \displaystyle\sum_{j=1}^{d} E\left[\log\left(\int_{\frac{-1}{2^{a_{j}}(\beta)}}^{\frac{1}{2^{a_{j}}(\beta)}}q(z_{j} + u|z_{.<j}, \beta) du\right)\right] \\
    \end{split}
\end{equation}
where $(a)$ uses the non-negativity of the Kullback-Leibler divergence $KL$ between the true distribution $p(z|\beta)$ and its approximation $q(z|\beta)$ once convolved with a uniform distribution over $[- 1 / 2^{a_{j}(\beta)}, 1 / 2^{a_{j}(\beta)}]$.

We follow \cite{salimans2017pixelcnn++} and for each $j=1, ..., d$ we model  $q(.|z_{.<j}, \beta)$ as a mixture of $K$ logistic distributions with means $\mu_{j, k}(\beta)$, scales $\gamma_{j, k}(\beta)$ and mixtures probability $\pi_{j, k}(\beta)$, which allows to compute exactly the integral term in \eqref{eq: entr2}. Specifically, we compute
\begin{equation}
    \mu_{j, k} = \mu_{j, k}^{0}(\beta) + w^{\mu}_{j, k}(\beta)\Gamma_{j}\odot z_{j},
\end{equation}
and 
\begin{equation}
    \log(\gamma_{j, k}) = \gamma_{j, k}^{0}(\beta) + w^{\gamma}_{j, k}(\beta)\Gamma_{j}\odot z_{j},
\end{equation}
where $\mu_{j, k}^{0}(.), \gamma_{j, k}^{0}()$ are functions from $[0,1]$ to $\mathbbm{R}$; $w^{\mu}_{j k}()$ and $w^{\gamma}_{j, k}()$ are functions from $[0,1]$ to $\mathbbm{R}^{d}$; and, $\Gamma_{j}=(1, 1, .., 1, 0, ...0)$ is a $d-$ dimensional vector equal to one for entry before $j$ and zero otherwise. $\Gamma_{j}$ guarantees that the distribution $q(.|z_{.<j})$ is conditioned only on $z_{.<j}$and not on any $z_{j^{'}}$ for $j^{'}\geq j$. 

The use of logistic distribution allows to compute the upper bound in \eqref{eq: entr2} as $H_{q}(z|\beta)$ where $H_{q}(z|\beta)$ is given by

\begin{equation}\nonumber
\begin{split}
- \displaystyle\sum_{j=1}^{d}E\left[\log\left(\displaystyle\sum_{k=1}^{K}\pi_{j, k}\sigma\left(\frac{z_{j} + \mu_{j, k}(\beta)}{\gamma_{j, k}(\beta)} + \frac{1}{2^{a_{j}(\beta)}}\right)\right.\right. &\\
 \left.\left. - \sigma\left(\frac{z_{j} + \mu_{j, k}(\beta)}{\gamma_{j, k}(\beta)} - \frac{1}{2^{a_{j}(\beta)}}\right)\right)\right]. &
 \end{split}
 \end{equation}

\section{Experimental Protocol}

\subsection{DSprites-Unfair}
We adapt DSprites\footnote{\url{https://github.com/deepmind/dsprites-dataset/}} to our fairness problem. DSprites (\cite{dsprites17}) contains $64$ by $64$ black and white images of various shapes (heart, square, circle), that are constructed from six independent factors of variation: color (black or white); shape (square, heart, ellipse), scales (6 values), orientation (40 angles in $[0, 2\pi]$); x- and y- positions (32 values each). The dataset results in $700K$ unique combinations of factor of variations. We modify the sampling to generate a source of potential unfairness. The sensitive attribute encodes which quadrant of the circle the orientation angle belongs to: $[0, \pi/2]$, $[\pi/2, \pi]$, $[\pi, 3/2\pi]$ and $[3/2\pi, 2\pi]$. All factors of variation but shapes are uniformly drawn. When sampling shapes, we assign to each possible combination of attributes a weight proportional to $1 + 10 \left[\left(\frac{i_{orientation}}{40})\right)^{3} + \left(\frac{i_{shape}}{3}\right)^{3}\right]$ , where $i_{shape}\in \{0, 1, 2\}$ and $i_{orientation}=\{0, 1, ..., 39\}$. This weight scheme generates a correlation between shapes and sensitive attributes (orientation).

\subsection{Adults / Heritage}
The \textbf{Adults} dataset \footnote{https://archive.ics.uci.edu/ml/datasets/adult}  contains $49K$ individuals with information to professional occupation, education attainment, capital gains, hours worked, race and marital status. We consider two variants of the Adults dataset: \textbf{Adults-Gender} and \textbf{Adults-Gender-Race}. In Adults-Gender we define as sensitive attributes the gender to which each individual self-identifies to.  In Adults-Gender-Race, we define as sensitive attribute an intersection of the gender and race an individual self-identifies to. For both Adults-Gender and Adults-Gender-Race, the downstream task label $Y$ correspond to whether individual income exceeds $50K$ per year.

The \textbf{Health Heritage} dataset \footnote{https://foreverdata.org/1015/index.html} contains $95K$ individuals with $65$ features related to clinical diagnoses and procedure, lab results, drug prescriptions and claims payment aggregated over $3$ years (2011-2013). We define as sensitive attributes a $18-$ dimensional variable that intersects the gender which individuals self-identify to and their reported age. 
The downstream task label $Y$ relates to whether an individual has a non-zero Charlson comorbidity Index, which is an indicator of a patient's 10-year survival rate.

\subsection{Comparative Methods}

\paragraph{LATFR:} This method ~\cite{madras2018learning} controls the mutual information $I(Z, S)$ via the cross-entropy of an adversary that predicts $S$ from $Z$. LATFR controls the fairness properties of the representation $Z$ with a single parameter in $\{ 0.1, 0.2, 0.3, 0.5, 0.7, 1.0, 1.5, 2.0, 3.0, 4.\}$ as prescribed in the original paper. 

\paragraph{Maxent-ARL:} This method \cite{roy2019mitigating} is a variant of LATFR that replaces the cross-entropy of the adversary with the entropy of is prediction. The fairness properties of $Z$ are controlled by a single parameter that we vary between $0$ and $1$ in steps of $0.1$, between $1$ and $10$ in steps of 1 and then between $10$ and $100$ in steps of 10.

\paragraph{$\beta-$VAE:}This method \cite{higgins2016beta,gitiaux2021fair} controls for $I(Z, S)$ by controlling the Kullback-Leibler divergence between $p(z)$ and an isotropic Gaussian prior. The fairness properties of the representations are controlled via the coefficient $\beta$ on the Kullback-Leibler divergence term: larger values of $\beta$ force $Z$ be more noisy, reduce the capacity of the channel between the data and the representation and thus, the mutual information $I(Z, X)$. We vary the value of $\beta$ between 0 and 1 in steps of 0.05.

\paragraph{CVIB:} This method ~\cite{moyer2018invariant} controls for $I(Z, S)$ via both the Kullback-Leibler divergence between $p(z)$ and an isotropic Gaussian prior (as in $\beta$-VAE) and an information-theory upper bound. The first term is controlled by a parameter $\beta$ that takes values in $\{0.001, 0.01, 0.1\}$; the second is controlled by a parameter $\lambda$ that vary between 0.01 to 0.1 in steps of 0.01 and 0.1 to 1.0 in steps of 0.1 (\cite{gupta2021controllable}).

\subsection{Architectures} 
\begin{table*}[hbpt]
    \centering
    \begin{tabular}{l|l|l|l}
    \toprule
    Dataset & Encoder & Decoder &Activation \\
    & & & \\
    \midrule
    DSprites & Conv(1, 32, 4, 2), Conv(32, 32, 4, 2) & Linear(28, 128), Linear(128, 1024) & ReLU \\
     & Conv(32, 64, 4, 2), Conv(64, 64, 4, 2) & ConvT2d(64, 64, 4, 2), ConvT2d(64, 32, 4, 2) &  \\
      & Linear(1024, 128) & ConvT2d(32, 32, 4, 2), ConvT2d(32, 61, 4, 2) &  \\
      \hline
      Adults & Linear(9, 128), Linear(128, 128)& Linear(8, 128), Linear(128, 128) & ReLU \\
       &  Linear(128, 8) & Linear(128, 9)  & \\
       \hline
Heritage       & Linear(65, 256), Linear(256, 256) & Linear(12, 256), Linear(256, 256)& ReLU \\
      & Linear(256, 12) & Linear(256, 65)  &  \\
       \bottomrule
\end{tabular}
\caption{\textbf{Architecture details.} $Conv2d(i, o, k, s)$ represents a 2D-convolutional layer with input channels $i$, output channels $o$, kernel size $k$ and stride $s$.  $ConvT2d(i, o, k, s)$ represents a 2D-deconvolutional layer with input channels $i$, output channels $o$, kernel size $k$ and stride $s$. $Linear(i, o)$ represents a fully connected layer with input dimension $i$ and output dimension $o$. Activations are not applied on the last layer of the decoder.}
\label{tab: arch}
\end{table*}

\begin{table}[h]
    \centering
    \begin{tabular}{l|l|l|l|l}
    \toprule
    Dataset & Number of iterations & Learning rate \\
    \midrule
    DSprites & 546K &  $0.3 \times 10^{-4}$ \\
      \midrule
      Adults & 27K & $0.3 \times 10^{-4}$ \\
       \midrule
       Heritage & 74K & $0.3 \times 10^{-4}$ \\
       \bottomrule
\end{tabular}
\caption{Hyperparameter values for SoFaiR / MSFaiR.}
\label{tab: hyper}
\end{table}

\paragraph{Encoder-decoders.} For the DSprites dataset, the autoencoder architecture -- taken directly from \cite{creager2019flexibly} -- includes $4$ convolutional layers and $4$ deconvolutional layers and uses ReLU activations. For Adults and Heritage, the encoder and decoder are made of fully connected layers with ReLU activations. Table \ref{tab: arch} shows more architectural details for each dataset. Moreover, means $\mu$, scales $\gamma$ and mixture probabilities $\pi$ are modeled as fully connected linear layers with input dimension $1$ and output dimension d, i.e. the dimension of the latent space. We choose $K=5$ logistic distributions in the mixture. We also set the maximum number of bits per dimension, $A$, to be equal to 8. Other hyperparameter values are in Table \ref{tab: hyper}.

\paragraph{Auditor and task classifiers.} Downstream classifiers and fairness auditors are multi-layer perceptrons with $2$ hidden layers of $256$ neurons each. Learning rates for both auditing and downstream tasks are set to $0.001$

\subsection{Additional Details on Area under unfairness-distortion curves.} 
We compute the area under unfair-distortion curve, AUFDC, as follows: 
\begin{itemize}
\item Since we only generate a finite number of points, empirical fairness-distortion curves do not have to exhibit a perfectly decreasing and smooth behavior.Therefore, to compute our AUFDC metric,  we first filter out the points on the curve that have higher distortion than points with higher $I(Z, S)$. That is, for any point $(D, I)$, we remove all points $(D^{'}, I^{'})$ for which $D^{'}>D$ and $I^{'}>I$.
\item We estimate the largest obtainable mutual information $I_{max}$ between $Z$ and $S$ as
\begin{equation}
H(S) - \min_{c}E_{s, x}[-\log c(s|x)],
\end{equation}
where $c:\mathcal{X}\rightarrow \mathcal{S}$ are classifiers that predict $S$  from the data $X$. That is, we use an adversarial estimate of $I(X, S)$ and use this estimate as $I_{max}$, since by the data processing inequality ~\cite{cover2012elements}, $I(Z, S)\leq I(X, S)$. 
\item For models that do not reach $I(Z, S)=0$, we compute the distortion $D_{max}$ obtained by generating random representations; and then, we add to the AUFDC score, the area $I_{min} \times (D_{max} - D_{min})$ where  $(D_{min}, I_{min})$ is the point the furthest on the bottom right of the unfairness-distiortion curve achieved by the model.
\item To allow comparison across datasets, we normalize the value of AUFD by the area of the rectangle $[0, D_{max}]\times[0, I_{max}]$.
\end{itemize}

\subsection{Pareto Front Protocol}
To generate the Pareto fronts in Figure 4 of the main text and Figure 2 of the supplemental file, we implement the following protocol:
\begin{itemize}
    \item Train an encoder-decoder architecture and freeze its parameters;
    \item Train an auditing classifier $c:\mathcal{Z}\rightarrow \mathcal{S}$ to predict $S$ from $Z$;
    \item Train a downstream task classifier $T:\mathcal{Z}\rightarrow\mathcal{Y}$ to predict a task $Y$ from $Z$.
\end{itemize}
The encoder-decoder does not access the task labels during training and our representation learning framework remains unsupervised with respect to downstream task labels.

\section{Additional Details on Quantization and Interpretability}

\subsection{Interpretability}
Each dimension $z_{j}(\beta)$ of the quantized representation is encoded into $a{j}(\beta)$ bits $b_{j,1}, b_{j, 2}, ..., b_{j, a(\beta)}$, where $b_{j, l}\in\{0, 1\}$ for $l=1, ..., a_{j}(\beta)$. For $\beta^{'}<\beta$ and for $j=1, ..., d$, we have
\begin{equation}
    z_{j}(\beta^{'})= z_{j}(\beta) + \displaystyle\sum_{l=a_{j}(\beta)}^{a_{j}(\beta^{'})}b_{j, l}2^{-l}.
\end{equation}
Therefore, we have a tractable and interpretable relation between $z_{j}(\beta^{'})$ and $z_{j}(\beta)$. In section 5.3 (RQ3), we compute  the first component of the term $ \displaystyle\sum_{l=a_{j}(\beta)}^{a_{j}(\beta^{'})}b_{j, l}2^{-l}$ and look at its correlation with each feature of the data $X$. 

\subsection{Entropy estimation.}
To estimate the entropy of the representation $Z$, we use an auto-regressive factorization to write the discrete distribution  $P(z|\beta)$ over the representation $Z$
\begin{equation}
    P(z|\beta) = \displaystyle\prod_{j=1}^{d}P(z_{j}|z_{.<j}, \beta),
\end{equation}
where the order of the dimension $j=1, ..., d$ is arbitrary and $z_{.<j}$ denotes the dimension between $1$ and $i-1$. This is similar to the entropy estimation in ~\cite{gitiaux2021fair}. However, unlike \cite{gitiaux2021fair} who model $P$ as a discrete distribution, we follow a more standard approach in rate-distortion \cite{theis2017lossy,choi2019variable} and approximate the discrete distribution $p(z_{j}|z_{.<j}, \beta)$ by a continuous distribution $q(z_{j}|z_{.<j}, \beta)$ such that the probability mass of $q$ on the interval $[z_{j} - 1 / 2^{a_{j}(\beta)}, z_{j} + 1 / 2^{a_{j}(\beta)}]$ is equal to $p(z_{j}|z_{.<j}, \beta)$. Therefore,
\begin{equation}
\label{eq: entr}
    \begin{split}
        H(z|\beta) & = -\displaystyle \sum_{j=1}^{d} E\left[\log p(z_{j}|z_{.<j}, \beta)\right] \\
        & = -\displaystyle \sum_{j=1}^{d} E\left[\log\left(\int_{-1/2^{a_{j}(\beta)}}^{1/2^{a_{j}(\beta)}}q(z_{j} + u|z_{.<j}, \beta) du\right)\right] \\
        & + KL\left(p||\int_{-1/2^{a_{j}(\beta)}}^{1/2^{a_{j}(\beta)}}q(z_{j} + u|z_{.<j}, \beta) du\right) \\
        & \overset{(a)}{\leq} -\displaystyle \sum_{j=1}^{d} E\left[\log\left(\int_{-1/2^{a_{j}(\beta)}}^{1/2^{a_{j}(\beta)}}q(z_{j} + u|z_{.<j}, \beta) du\right)\right], \\
    \end{split}
\end{equation}
where $(a)$ uses the non-negativity of the Kullback-Leibler divergence $KL$ between the true distribution $p(z|\beta)$ and its approximation $q(z|\beta)$ once convolved with a uniform distribution over $[- 1 / 2^{a_{j}(\beta)}, 1 / 2^{a_{j}(\beta)}]$.

We follow \cite{salimans2017pixelcnn++} and for each $j=1, ..., d$ we model  $q(.|z_{.<j}, \beta)$ as a mixture of $K$ logistic distributions with means $\mu_{j, k}(\beta)$, scales $\gamma_{j, k}(\beta)$ and mixtures probability $\pi_{j, k}(\beta)$, which allows to compute exactly the integral term in \eqref{eq: entr}. Specifically, we compute
\begin{equation}
    \mu_{j, k} = \mu_{j, k}^{0}(\beta) + w^{\mu}_{j, k}(\beta)\Gamma_{j}\odot z_{j},
\end{equation}
and 
\begin{equation}
    \log(\gamma_{j, k}) = \gamma_{j, k}^{0}(\beta) + w^{\gamma}_{j, k}(\beta)\Gamma_{j}\odot z_{j},
\end{equation}
where $\mu_{j, k}^{0}(.), \gamma_{j, k}^{0}()$ are functions from $[0,1]$ to $\mathbbm{R}$; $w^{\mu}_{j k}()$ and $w^{\gamma}_{j, k}()$ are functions from $[0,1]$ to $\mathbbm{R}^{d}$; and, $\Gamma_{j}=(1, 1, .., 1, 0, ...0)$ is a $d-$ dimensional vector equal to one for entry before $j$ and zero otherwise. $\Gamma_{j}$ guarantees that the distribution $q(.|z_{.<j})$ is conditioned only on $z_{.<j}$ only and not on any $z_{j^{'}}$ for $j^{'}\geq j$. 

The use of logistic distribution allows to compute the upper bound in \eqref{eq: entr} as $H_{q}(z|\beta)$ where $H_{q}(z|\beta)$ is given by

\begin{equation}
\begin{split}
- \displaystyle\sum_{j=1}^{d}E\left[\log\left(\displaystyle\sum_{k=1}^{K}\pi_{j, k}\sigma\left(\frac{z_{j} + \mu_{j, k}(\beta)}{\gamma_{j, k}(\beta)} + \frac{1}{2^{a_{j}(\beta)}}\right)\right.\right. &\\
 \left.\left. - \sigma\left(\frac{z_{j} + \mu_{j, k}(\beta)}{\gamma_{j, k}(\beta)} - \frac{1}{2^{a_{j}(\beta)}}\right)\right)\right]. &
 \end{split}
 \end{equation}
 
 The adaptive information bottleckneck can be written as:
\begin{equation}
\min_{g, h_{e}, h_{a}, \mu, \gamma, w, \pi}E[-\log p(x|g(z, s, \beta)) + \beta H_{q}(z|\beta)].
\end{equation}

\begin{figure*}
    \centering
    \includegraphics[width=1.0\textwidth]{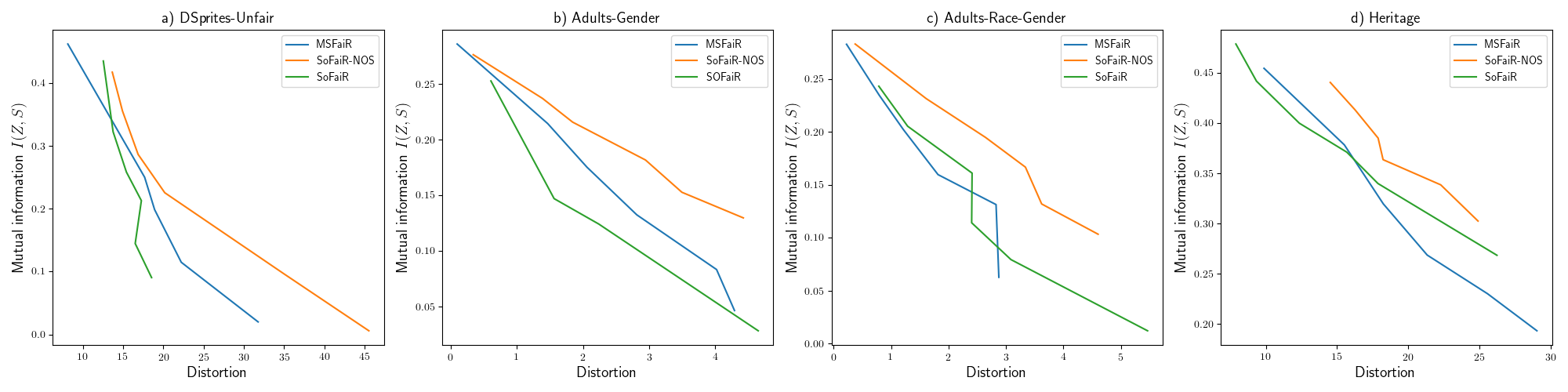}
    \caption{Ablation study for a) DSprites, b) Adults-Gender,  c) Adults-Race-Gender(left) and d) Heritage. This compares unfairness-distortion curves generates by our single shot approach SoFaiR to the ones generated by its multi-shot counterpart MSFaiR; and, to the ones generated by SoFaiR-NOS, which is similar to SoFaiR but for the decoder that does not receive the sensitive attribute $S$ as an input. }
    \label{fig:ablation}
\end{figure*}
\section{Additional Results}

\subsection{Ablation Study}
In Figure \ref{fig:ablation}, we plot the unfairness-distortion curves that correspond to the AUFDC that we report in Table 1 of the main text. We report the median value of distortion for a given level of mutual information $I(Z, S)$, where the median is taken over $10$ similar models trained with different seeds. The lower is the distortion for a given value of $I(Z, S)$, the better the fair representation learning method. Conclusions from Figure \ref{fig:ablation} are similar to the ones from Table 1. SoFaiR outperforms MSFaiR for Adults-Gender and Adults-Race-Gender at all values of $I(Z, S)$, while it is outperformed by MSFaiR for Heritage at low values of $I(Z, S)$. Moreover, for all datasets, SoFaiR outperforms SoFaiR-NOS, which confirms that rate-distortion solutions to fair representation learning need the decoder to receive the sensitive attribute as a side channel. In other words, the relation between unfairness-distortion and rate-distortion curves presented in Theorem 2.1 is tractable only if we use $E[-\log(p(x|z, s)]$ as a measure of distortion and would not necessarily hold if we were using  $E[-\log(p(x|z)]$ instead.

\subsection{Pareto Fronts for DSprites-Unfair and Heritage}
In Figure \ref{fig:pareto-front2}, we plot Pareto fronts for DSPrites-Unfair and Heritage. Both SoFaiR and MSFaiR outperform comparative  methods for DSPrites-Unfair and are competitive for Heritage. Moreover, on Heritage, SoFaiR is on par with MSFaiR, although it can generate the Pareto front with one trained model, while MSFaiR requires to re-train a model for each point of the front.

\begin{figure}[h]
    \centering
    \includegraphics[width=0.5\textwidth]{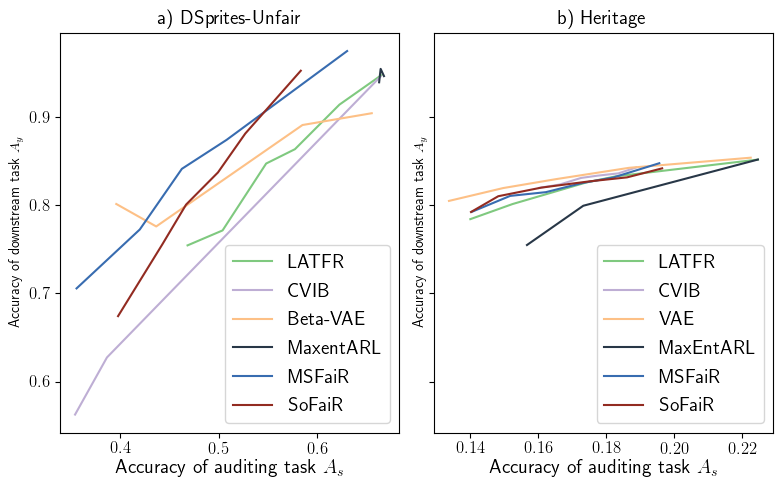}
    \caption{Pareto fronts for a) DSprites-Unfair (left),  b) Heritage (right). The downstream task label is whether income is larger than 50K. }
    \label{fig:pareto-front2}
\end{figure}

\subsection{Additional Plots for Bit Disparity}
In Figure \ref{fig:inter_adults2}, we show for Adults-Gender how moving up the unfairness-distortion curve (lower $\beta$) unmasks for each dimension of the representation additional bits and how these additional bits show higher disparity or correlation with respect to the sensitive attribute. 

\begin{figure*}
    \centering
    \includegraphics[width=0.95\textwidth]{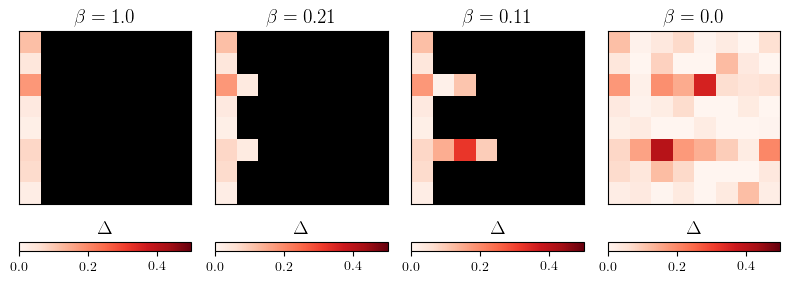}
    \caption{Unmasked bit stream for different values of the fairness coefficient $\beta$ for the Adults-Gender dataset. Rows represents a dimension of the representation $Z$. Each colored square represents an unmasked bit. Darker bits exhibit higher bit demographic disparity $\Delta$. Black squares represent masked bits. As $\beta$ decreases, SoFair unmasks more bits for each dimension of the representation $Z$. And, bits with lower disparity are more likely to be the first unmasked.}
    \label{fig:inter_adults2}
\end{figure*}

\pagebreak
\bibliographystyle{unsrt}
\bibliography{references}

\end{document}